\documentclass[11pt]{article}
\usepackage[T1]{fontenc}
\usepackage[utf8]{inputenc}
\usepackage{lmodern}
\usepackage[margin=1in]{geometry}
\usepackage{amsmath,amssymb,amsthm,mathtools,bm}
\usepackage{microtype}
\usepackage{enumitem,booktabs}
\usepackage[round,authoryear]{natbib}
\usepackage[hidelinks]{hyperref}
\usepackage{bookmark}
\usepackage{authblk}

\newtheorem{thm}{Theorem}[section]
\newtheorem{prop}[thm]{Proposition}

\newtheorem{cor}[thm]{Corollary}
\theoremstyle{definition}
\newtheorem{defn}[thm]{Definition}
\theoremstyle{remark}

\newenvironment{eqn}{\begin{equation}}{\end{equation}}
\numberwithin{equation}{section}
\newcommand{\R}{\mathbb R}

\newcommand{\E}{\mathbb E}
\newcommand{\PP}{\mathbb P}
\newcommand{\F}{\mathcal F}
\newcommand{\cL}{\mathcal L}
\newcommand{\ind}{\mathbf 1}
\newcommand{\topr}{^{\mathsf T}}
\newcommand{\op}{\mathrm{op}}
\DeclareMathOperator{\tr}{tr}
\DeclareMathOperator{\Cov}{Cov}
\DeclareMathOperator{\Var}{Var}
\DeclareMathOperator{\KL}{D_{KL}}
\DeclareMathOperator*{\argmax}{arg\,max}

\newcommand{\pto}{\xrightarrow{\PP}}

\setlist[itemize]{leftmargin=*,itemsep=0.55em,topsep=0.5em}
\setlist[enumerate]{leftmargin=*,itemsep=0.35em}
\allowdisplaybreaks[2]
\author[1]{Sudarshan Manikantan\thanks{manikantan.sudarshan@gmail.com}}
\author[2]{Abhishek Bhattacharjee\thanks{abhishek@theabstractmath.com}}

\affil[1,2]{Abstract Math Institute}
\affil[1]{\texttt{manikantan.sudarshan@gmail.com}}
\affil[2]{\texttt{abhishek@theabstractmath.com}}
\date{}
\title{Adaptive Random Matrices in Gaussian Bandits:\\Spectral Universality and Selection-Induced Outliers}
\begin{document}
\maketitle
\begin{abstract}
Adaptive arm selection changes the distribution of the observations collected by a bandit algorithm, but it need not change their limiting empirical spectrum. We study Gaussian bandit designs in which the dimension and the number of observations grow proportionally. A quantitative coupling theorem compares the design generated by any causal selection rule with an independent Gaussian design. If the logarithm of the number of available arms is sublinear in the dimension, the empirical spectral distribution converges to the Marchenko--Pastur law, uniformly over the selection rule. Consequently, Gaussian Bayesian bandits have policy-independent first-order limits for posterior mean-square uncertainty, squared posterior covariance, and information acquisition. For linear-score selection, we obtain the exact conditional arm distribution and show that two-arm selection produces an exactly Wishart Gram matrix in every dimension, despite its nonzero conditional mean. For a fixed selection direction, we identify an explicit eigenvalue and eigenvector transition governed by the second moment of a Gaussian maximum. A counterexample shows that a direction's overlap with a reference signal does not determine this transition. Finally, an exponentially large arm pool permits a different bulk limit, establishing the order-sharpness of the arm-growth condition. These results distinguish global spectral stability from directional effects in adaptive bandit data.
\end{abstract}

\section{Introduction}\label{sec:intro}
A linear bandit observes the reward of a selected covariate while the selection rule evolves with the accumulated data. The resulting Gram matrix is therefore generated jointly by fresh randomness and the algorithm's previous decisions. Its spectrum determines several statistical features of the experiment, including regularized estimation error and, under a Gaussian model, posterior uncertainty. Understanding this matrix is a problem in adaptive random matrix theory, not merely a consequence of a regret inequality.

The proportional regime is particularly informative. When the number of observations is comparable to the dimension, the eigenvalues of an independent Gaussian Gram matrix do not concentrate at a single point. A nontrivial empirical spectral distribution remains, and replacing the matrix by its expectation loses first-order information. Adaptivity introduces an additional difficulty: future observations depend on earlier rewards, so the columns of the selected design are generally dependent and non-Gaussian.

This paper establishes a distinction between two effects of selection. The first concerns the empirical distribution of all eigenvalues. The second concerns a small number of distinguished eigenvalues and directions. For a subexponential number of Gaussian candidates, the first effect vanishes at leading order for every causal policy. The second can persist. Thus, a selection-induced directional deformation does not by itself imply a different limiting bulk distribution.

The principal argument is based on the limited distributional change available to a selector. The conditional law of a vector selected from a pool of independent candidates is dominated by the pool size times the law of one candidate. Gaussian transportation converts this domination into a dimension-normalized quadratic coupling cost. Summing the costs over an adaptive experiment gives a comparison with an independent Gaussian matrix, without requiring convergence of the policy or independence of the selected columns.

For Bayesian bandits, this comparison gives explicit posterior uncertainty limits under Thompson sampling, greedy selection, and every other admissible policy. The result concerns statistical information, rather than equality of rewards. Two policies may collect the same first-order amount of information while using it differently when choosing an arm.

The directional analysis gives a complementary conclusion. Maximization of a Gaussian linear score has an exact decomposition into a one-dimensional Gaussian maximum and an independent orthogonal Gaussian component. The raw second moment can contain a rank-one deformation, whereas the centered covariance has a different coefficient. Under a fixed direction, the raw second moment determines an explicit outlier transition. Under two-arm selection, however, the entire outer-product distribution agrees with its Gaussian counterpart. These exact facts rule out several inferences based solely on the appearance of a preferred direction.

\section{Existing work}\label{sec:related}
The analysis draws on random matrix theory, adaptive bandit inference, and Gaussian transportation, while addressing a different asymptotic object from each of these literatures.

\paragraph{Sample covariance limits and regularized estimation.}
\citet{MP1967} identified the limiting spectral distribution of large sample covariance matrices. The spectral and resolvent theory developed in this direction is treated systematically by \citet{BaiSilverstein2010}. These results describe the independent Gaussian matrix used as a comparison object below; they do not directly apply to a matrix whose columns are chosen using previous rewards. In high-dimensional statistics, \citet{DW2018} derive exact proportional limits for ridge prediction and classification through spectral transforms. The posterior trace limits in the present paper have the same independent-design reference law, but the main issue is establishing that this law remains valid uniformly over adaptive selectors.

\paragraph{Transportation and matrix comparison.}
\citet{Talagrand1996} bounds quadratic transportation cost by relative entropy for Gaussian measures. We apply that inequality to conditional selected-arm distributions and concatenate the resulting couplings along the experiment. The subsequent comparison of Gram matrices uses standard unitarily invariant norm inequalities; see \citet{Bhatia1997}. The new conclusion obtained from these tools is a policy-uniform spectral comparison under a bound on the logarithm of the candidate-pool size. Neither the transportation inequality nor the independent-design Marchenko--Pastur law alone states this adaptive conclusion.

\paragraph{Spiked random matrices.}
The eigenvalue transition for a finite-rank population deformation is central to the work of \citet{BBP2005}. Almost-sure outlier locations for spiked sample covariance matrices are characterized by \citet{BaikSilverstein2006}, and \citet{Paul2007} studies the associated sample eigenvectors. Our fixed-direction theorem has the corresponding first-order transition, with a spike parameter determined by a Gaussian maximum. The selected row is not centered Gaussian. Its reduction uses independence of the orthogonal Gaussian block and right-orthogonal invariance. We prove the reduction and the resulting first-order formulas; no fluctuation theorem at the spectral edge is asserted.

\paragraph{Bandit algorithms and adaptive covariance geometry.}
\citet{DM2013} study Gaussian-prior linear bandits in a data-poor regime and obtain matching-order cumulative reward bounds under geometric conditions on the available action set. Their work makes the high-dimensional learning regime an established bandit problem, rather than a new motivation by itself. \citet{AG2013} establish regret guarantees for a Gaussian-sampling algorithm in linear contextual bandits. \citet{Russo2018} review posterior sampling and its decision-theoretic interpretation, while \citet{RV2016} relate Thompson-sampling regret to information acquisition. The object studied here is instead the proportional-dimensional spectrum of the collected design. A particularly relevant comparison is \citet{Fan2025}, who analyze the eigenvalues and eigenvectors generated by LinUCB on the Euclidean unit ball and obtain inferential consequences. Their action space is a continuum under the learner's control. Our action space consists of fresh independent Gaussian candidates. This distinction is essential: the pool-size restriction supplies the distributional domination used in our theorem, and an unrestricted unit ball does not satisfy the same condition.

\paragraph{Exploration from random contexts and matrix-valued bandits.}
Sufficient contextual variation can make explicit exploration unnecessary. \citet{Bastani2021} develop exploration-free guarantees under covariate diversity, \citet{Kannan2018} analyze greedy selection under smoothed contexts, and \citet{Raghavan2023} compare batched greedy policies with other algorithms in Bayesian regret. Our results give a spectral statement for every causal policy, including policies not chosen for low regret. They do not replace those regret analyses. Low-rank matrix bandits address another problem: the unknown parameter itself has matrix structure. For example, \citet{Jang2024} combine low-rank estimation and experimental design to obtain regret bounds. Here, the parameter may be an ordinary vector; the random matrix is the design produced by the bandit experiment.

\section{Our contributions}\label{sec:contributions}
The results provide a quantitative separation between the global spectrum, posterior information, and directional concentration of Gaussian bandit data.
\begin{itemize}
\item \textbf{A comparison theorem uniform over all causal selectors.}
Theorem~\ref{thm:coupling} constructs an independent Gaussian comparison design for an arbitrary adaptive policy and bounds its quadratic coupling cost by $2n\log K/d$. It also bounds the normalized nuclear-norm difference between the two Gram matrices. The theorem requires no prescribed selection direction, exploration schedule, stability condition, or limiting allocation law.
\item \textbf{A policy-independent proportional spectral limit.}
Theorem~\ref{thm:bulk} proves convergence to the Marchenko--Pastur law whenever $n/d$ has a finite positive limit and $\log K=o(d)$. The convergence holds in expected Wasserstein distance uniformly over policies. Corollary~\ref{cor:resolvent} supplies the corresponding resolvent equation and explicit regularized inverse trace. Adaptive selection can therefore change individual directions without changing these first-order global spectral quantities.
\item \textbf{Explicit Bayesian uncertainty and information limits.}
Theorem~\ref{thm:posterior} transfers the comparison to posterior covariance, realized squared estimation error, and mutual information. The posterior trace converges uniformly over compact proportional-time intervals and uniformly over policies. Corollary~\ref{cor:posterior-normal} also gives a conditional normal approximation for squared posterior error, separating Bayesian uncertainty quantification from fixed-parameter frequentist claims.
\item \textbf{An exact two-arm identity, not only an asymptotic approximation.}
Proposition~\ref{prop:winner} identifies the full selected-arm law. Corollary~\ref{cor:two} shows that, for two candidates, its outer product has precisely the Gaussian outer-product law. Consequently, every predictable linear-score policy, including Thompson sampling, produces a Wishart Gram matrix at every finite dimension and horizon. The selected vector nevertheless has a nonzero conditional mean.
\item \textbf{A completely specified selection-induced outlier transition.}
Theorem~\ref{thm:outlier} gives the largest eigenvalue, the second eigenvalue, and the leading eigenvector overlap under a fixed direction. The threshold is $\gamma\beta_K^2=1$, where $\beta_K$ is defined in Definition~\ref{def:winner}. The theorem concerns a static directional experiment and does not presume that a fully adaptive Thompson-sampling trajectory behaves as a fixed spike.
\item \textbf{Two precise limits on extrapolation.}
Proposition~\ref{prop:overlap} constructs policies with identical overlap curves and different largest-eigenvalue limits. Theorem~\ref{thm:exponential} constructs an exponentially large arm pool whose selected Gram matrix has a rescaled, rather than unchanged, Marchenko--Pastur bulk. These results identify information absent from an overlap curve and establish the order-sharpness of the subexponential arm-growth condition for policy-uniform bulk universality.
\end{itemize}
\paragraph{Organization of the paper.}
Section~\ref{sec:setup} defines the adaptive experiment and its spectral objects. Section~\ref{sec:main} presents the coupling, bulk, selected-arm, and posterior results. Section~\ref{sec:outliers} studies fixed-direction outliers and the insufficiency of an overlap curve. Section~\ref{sec:largepool} treats exponentially many arms. Section~\ref{sec:discussion} discusses the scope of the conclusions. All proofs appear in Appendices~\ref{app:main}--\ref{app:largepool}.

\section{Adaptive Gaussian bandit designs}\label{sec:setup}
This section fixes the information available at selection time and the normalization of the random matrices.

\begin{defn}[Adaptive experiment]\label{def:experiment}
For a dimension $d\geq1$ and an integer $K\geq1$, round $t$ presents candidates $x_{t,1},\ldots,x_{t,K}$ that are independent $N(0,I_d/d)$ vectors, independent of the history $\F_{t-1}$. A fresh random seed $U_t$, independent of the candidates and of all previously generated variables, is available to the policy. Set $\mathcal H_t=\F_{t-1}\vee\sigma(U_t)$. An admissible policy chooses $A_t\in\{1,\ldots,K\}$ measurably with respect to $\mathcal H_t\vee\sigma(x_{t,1},\ldots,x_{t,K})$. The next history contains the current seed, candidates, action, and any feedback received after selection. Every subsequent candidate pool is independent of the entire past. Histories and seeds take values in standard Borel spaces. The collection of admissible policies is denoted by $\Pi_{d,K}$.

Write $x_t=x_{t,A_t}$, $X_n=[x_1\ \cdots\ x_n]$, and $G_n=X_nX_n\topr$, with $G_0=0$. The number of arms may depend on dimension, in which case it is denoted by $K_d$.
\end{defn}
The spectral results allow arbitrary feedback mechanisms compatible with Definition~\ref{def:experiment}. The Bayesian specialization below fixes a particular mechanism. The policy never has access to a future candidate pool.

\begin{defn}[Spectral and transportation quantities]\label{def:spectral}
For a positive semidefinite $d\times d$ matrix $B$, let $\mu_B=d^{-1}\sum_{j=1}^d\delta_{\lambda_j(B)}$, with eigenvalues in decreasing order. For probability measures $\nu,\eta$ on a Euclidean space, $W_p(\nu,\eta)^p$ is the infimum of $\int\|x-y\|^p\,dQ(x,y)$ over couplings $Q$ of $\nu$ and $\eta$, for $p\in\{1,2\}$. Relative entropy is $\KL(\nu\|\eta)=\int\log(d\nu/d\eta)\,d\nu$ when $\nu\ll\eta$, and is infinite otherwise. Matrix norms $\|\cdot\|_*$, $\|\cdot\|_F$, and $\|\cdot\|_{\op}$ denote the nuclear, Frobenius, and operator norms.
\end{defn}

\begin{defn}[Reference spectral law]\label{def:MP}
For $\gamma>0$, put $a_\gamma=(1-\sqrt\gamma)^2$ and $b_\gamma=(1+\sqrt\gamma)^2$, and define
\begin{eqn}\label{eq:MP}
\mu_\gamma=(1-\gamma)_+\delta_0+
\frac{\sqrt{(b_\gamma-x)(x-a_\gamma)}}{2\pi x}
\ind_{[a_\gamma,b_\gamma]}(x)\,dx.
\end{eqn}
Set $\mu_0=\delta_0$. For $\operatorname{Im}z>0$ and $\lambda>0$, respectively, define $m_\gamma(z)=\int(x-z)^{-1}\,d\mu_\gamma(x)$ and $h_\gamma(\lambda)=\int(x+\lambda)^{-1}\,d\mu_\gamma(x)$.
\end{defn}
This is the normalization appropriate to $d^{-1}ZZ\topr$ when $Z$ has $n$ independent standard Gaussian columns and $n/d\to\gamma$. In particular, its first moment is $\gamma$, not one.

\begin{defn}[Linear-score selection]\label{def:winner}
A linear-score policy chooses an $\mathcal H_t$-measurable unit vector $u_t$ before the current candidates and selects $A_t=\argmax_{a\leq K}\langle x_{t,a},u_t\rangle$. Ties are resolved by the smallest index. For independent standard normal variables $Z_1,\ldots,Z_K$, put $M_K=\max_{a\leq K}Z_a$, $\mu_K=\E M_K$, and $\beta_K=\E M_K^2-1$. The standard normal density and distribution function are denoted by $\phi$ and $\Phi$.
\end{defn}

\begin{defn}[Gaussian Bayesian bandit]\label{def:Bayes}
Fix $\tau,\sigma>0$. Independently of the candidate pools and policy seeds, let $\theta^\star\sim N(0,\tau^2I_d)$ and let $\varepsilon_t$ be independent $N(0,\sigma^2)$ variables. The feedback is $Y_t=x_t\topr\theta^\star+\varepsilon_t$. The initial observed history is trivial, and subsequent observed histories are generated by the seeds, candidate pools, actions, and rewards. Policies are functions only of their specified observations and independent randomization. Define
\begin{eqn}\label{eq:posterior-objects}
C_t=(\tau^{-2}I_d+\sigma^{-2}G_t)^{-1},
\qquad m_t=\sigma^{-2}C_t\sum_{s=1}^t x_sY_s.
\end{eqn}
For $s\geq0$, define the deterministic functions
\begin{eqn}\label{eq:posterior-limits}
\begin{split}
v(s)&=\frac{\sqrt{\{\sigma^2+\tau^2(s-1)\}^2+4\sigma^2\tau^2}
-\{\sigma^2+\tau^2(s-1)\}}{2},\\
w(s)&=\sigma^4\int\frac{d\mu_s(x)}{(x+\sigma^2/\tau^2)^2},\qquad
j(s)=\frac12\int\log(1+\tau^2x/\sigma^2)\,d\mu_s(x).
\end{split}
\end{eqn}
Mutual information with the observed history means
$I(\theta^\star;\F_t)=\E\KL\{\cL(\theta^\star\mid\F_t)\|\cL(\theta^\star)\}$.
\end{defn}

\section{Main results}\label{sec:main}
The main results first compare arbitrary adaptive designs, then identify the additional structure available under linear-score selection and Gaussian rewards.

\subsection{A quantitative comparison with independent Gaussian data}
\begin{thm}[Adaptive Gaussian coupling]\label{thm:coupling}
Under Definition~\ref{def:experiment}, for every $n$ and every $\pi\in\Pi_{d,K}$, an extension of the probability space supports independent vectors $z_1,\ldots,z_n\sim N(0,I_d/d)$ such that, with $Z_n=[z_1\ \cdots\ z_n]$, all prefixes $1\leq t\leq n$ satisfy
\begin{eqn}\label{eq:coupling-cost}
\E\|X_t-Z_t\|_F^2\leq\frac{2t\log K}{d}.
\end{eqn}
The same coupling satisfies
\begin{eqn}\label{eq:nuclear-bound}
\E\max_{0\leq t\leq n}\frac{\|G_t-Z_tZ_t\topr\|_*}{d}
\leq\frac{2n}{d}\left\{\sqrt{\frac{2\log K}{d}}+\frac{\log K}{d}\right\}.
\end{eqn}
In particular, the right side of \eqref{eq:nuclear-bound} bounds
$\E W_1(\mu_{G_n},\mu_{Z_nZ_n\topr})$ and, after multiplication by $L$, bounds
$\E|d^{-1}\tr f(G_n)-d^{-1}\tr f(Z_nZ_n\topr)|$ for every $L$-Lipschitz function $f$ on $[0,\infty)$.
\end{thm}
The vectors $z_t$ are independent of one another; they are not asserted to be independent of the selected design. Such independence would be incompatible with the purpose of the coupling.

\begin{thm}[Policy-uniform empirical spectral law]\label{thm:bulk}
Under Definitions~\ref{def:experiment}--\ref{def:MP}, if $n_d/d\to\gamma\in(0,\infty)$ and $\log K_d=o(d)$, then
\begin{eqn}\label{eq:uniform-bulk}
\sup_{\pi\in\Pi_{d,K_d}}\E W_1(\mu_{G_{n_d}},\mu_\gamma)\longrightarrow0.
\end{eqn}
Consequently, every sequence of admissible policies has the same limiting empirical spectral distribution and the same limiting normalized trace for every fixed Lipschitz spectral function.
\end{thm}

\begin{cor}[Resolvents and inverse traces]\label{cor:resolvent}
Under the assumptions of Theorem~\ref{thm:bulk}, for each $z$ with $\operatorname{Im}z>0$ and each $\lambda>0$,
\[
\sup_\pi\E\left|\frac1d\tr(G_{n_d}-zI_d)^{-1}-m_\gamma(z)\right|\to0,
\qquad
\sup_\pi\E\left|\frac1d\tr(G_{n_d}+\lambda I_d)^{-1}-h_\gamma(\lambda)\right|\to0.
\]
The reference transforms satisfy
\begin{eqn}\label{eq:transform-equations}
z m_\gamma(z)^2+(z+1-\gamma)m_\gamma(z)+1=0,
\qquad
\lambda h_\gamma(\lambda)^2+(\lambda+\gamma-1)h_\gamma(\lambda)-1=0.
\end{eqn}
The first solution is selected by $\operatorname{Im}m_\gamma(z)>0$ and $zm_\gamma(z)\to-1$ at infinity. The second is the unique positive solution,
\[
h_\gamma(\lambda)=
\frac{\sqrt{(\lambda+\gamma-1)^2+4\lambda}-(\lambda+\gamma-1)}{2\lambda}.
\]
\end{cor}
The theorem is a bulk statement. The metric in \eqref{eq:uniform-bulk} does not control a single extreme eigenvalue at order one.

\subsection{Exact selected-arm laws}
\begin{prop}[Gaussian maxima and orthogonal noise]\label{prop:winner}
Under Definition~\ref{def:winner}, conditionally on $\mathcal H_t$,
\begin{eqn}\label{eq:winner-representation}
\sqrt d\,x_t\ \stackrel{\mathrm d}{=}\ M_Ku_t+(I_d-u_tu_t\topr)g,
\end{eqn}
where $g\sim N(0,I_d)$ and $M_K$ are independent. The conditional density of $\sqrt d\,x_t$ at $y$ is
$K\phi_d(y)\Phi(u_t\topr y)^{K-1}$, where $\phi_d$ is the standard $d$-dimensional Gaussian density. In particular,
\[
\E(\sqrt d\,x_t\mid\mathcal H_t)=\mu_Ku_t,
\quad
\E(d x_tx_t\topr\mid\mathcal H_t)=I_d+\beta_Ku_tu_t\topr,
\quad
\Cov(\sqrt d\,x_t\mid\mathcal H_t)=I_d+(\beta_K-\mu_K^2)u_tu_t\topr.
\]
Moreover, $\beta_1=\beta_2=0$, and, for $K\geq3$,
\begin{eqn}\label{eq:beta-integral}
\beta_K=\frac{K(K-1)(K-2)}2\int_{\R}\phi(z)^3\Phi(z)^{K-3}\,dz>0.
\end{eqn}
For example, $\mu_2=\pi^{-1/2}$ and $\beta_3=\sqrt3/(2\pi)$.
\end{prop}
The rank-one coefficient in the raw second moment is not the coefficient in the centered covariance. This distinction matters because $G_t$ is uncentered.

\begin{cor}[Exact Wishart law for two-arm selection]\label{cor:two}
For $K=2$ under Definition~\ref{def:winner}, the matrices $d x_tx_t\topr$ are independent and each has the law of $gg\topr$ for $g\sim N(0,I_d)$. Thus, for every $d,n$, $dG_n$ has the Wishart distribution with identity scale and $n$ degrees of freedom. The conclusion also holds for $K=1$.
\end{cor}
The assertion concerns outer products, not the selected vectors. In particular, a nonzero conditional mean is consistent with this exact Wishart identity.

\subsection{Posterior uncertainty in Gaussian bandits}
\begin{thm}[Policy-independent posterior limits]\label{thm:posterior}
Under Definition~\ref{def:Bayes}, $\cL(\theta^\star\mid\F_t)=N(m_t,C_t)$ for every admissible policy. If $\log K_d=o(d)$, then, for every $\Gamma<\infty$,
\begin{eqn}\label{eq:uniform-posterior}
\sup_{\pi\in\Pi_{d,K_d}}\E\sup_{0\leq s\leq\Gamma}
\left|\frac1d\tr C_{\lfloor sd\rfloor}-v(s)\right|\longrightarrow0.
\end{eqn}
For $n_d/d\to\gamma\in(0,\infty)$,
\[
\sup_\pi\E\left|\frac1d\tr C_{n_d}^2-w(\gamma)\right|\to0,
\qquad
\sup_\pi\E\left|\frac1d\|\theta^\star-m_{n_d}\|^2-v(\gamma)\right|\to0,
\]
and
\begin{eqn}\label{eq:information}
I(\theta^\star;\F_n)=\frac12\E\log\det(I_d+\tau^2G_n/\sigma^2),
\qquad
\sup_\pi\left|\frac1d I(\theta^\star;\F_{n_d})-j(\gamma)\right|\to0.
\end{eqn}
The function $v$ is strictly positive, satisfies $v(0)=\tau^2$, and is the unique positive solution of
\begin{eqn}\label{eq:v-polynomial}
v(s)^2+\{\sigma^2+\tau^2(s-1)\}v(s)-\sigma^2\tau^2=0.
\end{eqn}
\end{thm}

\begin{cor}[Conditional normal approximation for posterior error]\label{cor:posterior-normal}
Under Definition~\ref{def:Bayes}, a universal finite constant $C$ satisfies, almost surely,
\[
\sup_{z\in\R}\left|
\PP\left(\left.
\frac{\|\theta^\star-m_n\|^2-\tr C_n}{\sqrt{2\tr C_n^2}}\leq z
\,\right|\F_n\right)-\Phi(z)\right|
\leq C\frac{\tr C_n^3}{(\tr C_n^2)^{3/2}}.
\]
Under the proportional assumptions of Theorem~\ref{thm:posterior}, the right side is $O_\PP(d^{-1/2})$ for every policy sequence.
\end{cor}
This is a conditional posterior statement and a joint Bayesian coverage statement. It does not assert conditional-on-$\theta^\star$ frequentist coverage for a fixed unknown parameter.

\section{Directional outliers and the limits of overlap information}\label{sec:outliers}
This section isolates the directional phenomenon in an experiment where its assumptions can be stated and verified exactly.

\begin{thm}[Fixed-direction eigenvalue and eigenvector transition]\label{thm:outlier}
Under Definition~\ref{def:winner}, suppose $K$ is fixed, $u_t=u_d$ for all $t$, where $u_d$ is a deterministic unit vector, and $n_d/d\to\gamma\in(0,\infty)$. Let $v_1(G_{n_d})$ be a unit leading eigenvector. If $\beta_K\sqrt\gamma\leq1$, then
\[
\lambda_1(G_{n_d})\pto b_\gamma,
\qquad |u_d\topr v_1(G_{n_d})|^2\pto0.
\]
If $\beta_K\sqrt\gamma>1$, then
\[
\lambda_1(G_{n_d})\pto(1+\beta_K)(\gamma+\beta_K^{-1}),
\qquad
|u_d\topr v_1(G_{n_d})|^2\pto
\frac{1-(\gamma\beta_K^2)^{-1}}{1+(\gamma\beta_K)^{-1}}.
\]
In both cases $\lambda_2(G_{n_d})\pto b_\gamma$. In particular, $K=3$ has threshold $\gamma=4\pi^2/3$, whereas $K\leq2$ has no fixed-direction upper outlier.
\end{thm}

\begin{prop}[Identical overlaps, different spectral extremes]\label{prop:overlap}
Fix $K\geq3$ and $\gamma>\beta_K^{-2}$. For $d\geq2$, let the reference direction be $u_d^\star=e_1$. There exist two deterministic linear-score policies for which $|u_t\topr u_d^\star|^2=0$ at every round, but whose Gram matrices satisfy, respectively,
\[
\lambda_1(G_{\lfloor\gamma d\rfloor})\pto
(1+\beta_K)(\gamma+\beta_K^{-1})
\quad\hbox{and}\quad
\lambda_1(G_{\lfloor\gamma d\rfloor})\pto b_\gamma.
\]
Both empirical spectral distributions converge to $\mu_\gamma$.
\end{prop}
Thus, an overlap curve relative to a single reference signal does not characterize the persistence of directions orthogonal to that signal. A scalar overlap assumption cannot, by itself, justify a general adaptive outlier theorem.

\section{Exponentially many arms can change the bulk}\label{sec:largepool}
This section shows that the scale of the pool-size condition in Theorem~\ref{thm:bulk} cannot be increased uniformly over policies.

\begin{defn}[Largest-norm selection]\label{def:norm}
Fix $\alpha>0$, set $K_d=\lceil e^{\alpha d}\rceil$, and choose the candidate of largest Euclidean norm at each round. Let $r_\alpha>1$ be the unique solution of
\[
\frac{r_\alpha-1-\log r_\alpha}{2}=\alpha.
\]
For a probability measure $\nu$ on $[0,\infty)$, $(r_\alpha\cdot)_\#\nu$ denotes its image under $x\mapsto r_\alpha x$.
\end{defn}

\begin{thm}[A different bulk at exponential pool size]\label{thm:exponential}
Under Definition~\ref{def:norm}, if $n_d/d\to\gamma\in(0,\infty)$, then
\[
\E W_1\bigl(\mu_{G_{n_d}},(r_\alpha\cdot)_\#\mu_\gamma\bigr)\longrightarrow0.
\]
Since $r_\alpha>1$, this limiting law differs from $\mu_\gamma$.
\end{thm}
The theorem establishes sharpness in the order of $\log K_d$, not a necessary condition for every individual policy. Some policies retain the Gaussian bulk even with a larger available pool because they do not use that additional choice.

\section{Discussion}\label{sec:discussion}
The spectral effect of Gaussian bandit selection depends on the scale of the statistic being examined. A bounded or subexponentially growing pool cannot change the empirical spectral law at leading order, irrespective of the feedback-driven policy. Nevertheless, a persistent direction can produce an order-one outlier, and the location and alignment of that outlier are not recoverable from the bulk law.

The Bayesian consequences are equally specific. All admissible policies have the same first-order posterior mean-square uncertainty and mutual information in the proportional regime. Equality of these learning quantities does not imply equality of action quality. A policy may use its posterior mean, a posterior draw, or no reward information at all while collecting the same leading-order amount of information. The distinction between statistical learning and immediate reward is therefore intrinsic to this model.

The two-arm identity illustrates why conditional second moments require careful interpretation. Linear-score selection changes the mean of the chosen vector but preserves its entire outer-product distribution when there are two candidates. In this case, the Gram matrix is exactly Wishart, not merely asymptotically equivalent to Wishart. For three or more candidates, the additional second moment is visible in a persistent direction, while its contribution to normalized bulk statistics remains negligible.

The fixed-direction transition does not establish a learning-induced transition for the fully adaptive directions of Thompson sampling. Such a statement requires control of their temporal geometry, beyond a one-dimensional overlap law. Proposition~\ref{prop:overlap} makes this limitation explicit. The results also do not provide edge fluctuation laws, anisotropic resolvent limits for arbitrary policies, or second-order differences in posterior risk. These quantities are not controlled by normalized nuclear-norm convergence.

Finally, the assumptions on the candidate pool are substantive. Correlated candidates, a continuous action set, an exponentially large pool, or a design law without the Gaussian transportation property can change the comparison argument. Theorem~\ref{thm:exponential} gives an explicit change of bulk within the same Gaussian model when the selector has exponentially many candidates. The resulting distinction is between the information available through selection and the dimension of the observation, rather than between particular named algorithms.

\appendix
\section{Proofs of the main results}\label{app:main}
The proofs are arranged in the order of the statements in Section~\ref{sec:main}.

\subsection{Proof of Theorem~\ref{thm:coupling}}
\begin{proof}
Fix a finite horizon. Work recursively on an enlarged probability space, and condition at round $t$ on the original past and all Gaussian comparison variables already constructed. Fresh candidates are still independent of this enlarged past: previously introduced coupling variables are generated only from past observations and additional independent randomization.

Let $\nu_t$ be the conditional law of $\sqrt d\,x_t$, and let $\gamma_d$ be standard Gaussian measure on $\R^d$. For every Borel set $B$, selection of one of the candidates implies
\[
\nu_t(B)\leq\sum_{a=1}^K
\PP(\sqrt d\,x_{t,a}\in B\mid\hbox{enlarged past})=K\gamma_d(B).
\]
Consequently $\nu_t\ll\gamma_d$, $d\nu_t/d\gamma_d\leq K$, and
$\KL(\nu_t\|\gamma_d)\leq\log K$. The Gaussian transportation inequality of \citet{Talagrand1996}, in standard-normal normalization, yields
\[
W_2(\nu_t,\gamma_d)^2\leq2\KL(\nu_t\|\gamma_d)\leq2\log K.
\]
The selected vector has a finite second moment because it is one of finitely many Gaussian vectors. Quadratic optimal couplings therefore exist. They can be chosen measurably as a function of the conditional law on these standard Borel spaces. Equivalently, measurable approximate coupling kernels can first be constructed from countable partitions, with the optimal bound obtained by taking a limit.

Disintegrate such a coupling with respect to its selected-vector marginal. After generating the actual round, sample a coupled Gaussian vector using that conditional kernel. Its conditional marginal, given the enlarged past, is always $\gamma_d$. Therefore the resulting Gaussian comparison vectors are independent across rounds. This construction preserves the original experiment: the added variables are not provided as inputs to the policy. After rescaling, it gives
$\E(\|x_t-z_t\|^2\mid\hbox{enlarged past})\leq2\log K/d$. Summation proves \eqref{eq:coupling-cost} for all prefixes.

Put $D_t=X_t-Z_t$. The identity
$G_t-Z_tZ_t\topr=Z_tD_t\topr+D_tZ_t\topr+D_tD_t\topr$ and the inequality $\|AB\topr\|_*\leq\|A\|_F\|B\|_F$ imply
\[
\|G_t-Z_tZ_t\topr\|_*
\leq2\|Z_t\|_F\|D_t\|_F+\|D_t\|_F^2
\leq2\|Z_n\|_F\|D_n\|_F+\|D_n\|_F^2.
\]
Since $\E\|Z_n\|_F^2=n$, Cauchy--Schwarz gives
\[
\E\max_{t\leq n}\frac{\|G_t-Z_tZ_t\topr\|_*}{d}
\leq\frac{2}{d}\sqrt{n\frac{2n\log K}{d}}
+\frac{2n\log K}{d^2},
\]
which is \eqref{eq:nuclear-bound}.

For Hermitian matrices, the sum of absolute differences of ordered eigenvalues is bounded by the nuclear norm of the difference; see \citet{Bhatia1997}. Pairing these eigenvalues gives the asserted $W_1$ bound. Applying a Lipschitz function to the same pairing proves the trace comparison.
\end{proof}

\subsection{Proof of Theorem~\ref{thm:bulk}}
\begin{proof}
For the independent Gaussian comparison matrix, the Marchenko--Pastur theorem gives weak convergence of $\mu_{Z_{n_d}Z_{n_d}\topr}$ to \eqref{eq:MP}. Its first moment is
\[
\frac1d\tr(Z_{n_d}Z_{n_d}\topr)
=\frac1{d^2}\sum_{i=1}^d\sum_{t=1}^{n_d}g_{it}^2,
\]
where $g_{it}$ are independent standard normals. This converges in $L^2$ to $\gamma$, the first moment of $\mu_\gamma$. Weak convergence together with convergence of first moments yields $W_1$ convergence in probability. Moreover, $W_1$ is bounded by the sum of the two first moments, which is uniformly bounded in $L^2$. Thus the convergence also holds in expectation.

The triangle inequality and Theorem~\ref{thm:coupling} now give
\[
\sup_\pi\E W_1(\mu_{G_{n_d}},\mu_\gamma)
\leq\frac{2n_d}{d}
\left\{\sqrt{\frac{2\log K_d}{d}}+\frac{\log K_d}{d}\right\}
+\E W_1(\mu_{Z_{n_d}Z_{n_d}\topr},\mu_\gamma).
\]
Both terms vanish. The argument uses only the marginal law of the comparison matrix, which is identical for every policy, so the supremum is justified.
\end{proof}

\subsection{Proof of Corollary~\ref{cor:resolvent}}
\begin{proof}
On $[0,\infty)$, the function $x\mapsto(x-z)^{-1}$ is Lipschitz with constant at most $(\operatorname{Im}z)^{-2}$, and $x\mapsto(x+\lambda)^{-1}$ is Lipschitz with constant $\lambda^{-2}$. The convergence follows from Theorem~\ref{thm:bulk}. The Stieltjes transform of \eqref{eq:MP} satisfies the first equation in \eqref{eq:transform-equations}; this is the Marchenko--Pastur transform identity in the present normalization. Its sign and asymptotic conditions select the probability-measure branch. Substituting $z=-\lambda$ gives the second equation. Its constant term is negative and its leading coefficient positive, so exactly one root is positive; the displayed formula is that root.
\end{proof}

\subsection{Proof of Proposition~\ref{prop:winner}}
\begin{proof}
Condition on $\mathcal H_t$ and abbreviate its fixed direction by $u$. Decompose each candidate as
$\sqrt d\,x_{t,a}=Z_au+g_{a,\perp}$. The variables $Z_a$ are independent standard normals, the $g_{a,\perp}$ are independent $N(0,I_d-uu\topr)$ vectors, and the two families are independent. The maximizing index is a function only of $(Z_1,\ldots,Z_K)$. Conditional on that family, the orthogonal vector at the maximizing index still has law $N(0,I_d-uu\topr)$, independent of the maximum. This proves \eqref{eq:winner-representation}.

A given candidate with standardized value $y$ is selected precisely when every other score is at most $u\topr y$. Summing over its $K$ possible indices gives the stated density. Independence of the two components yields the three moment identities, including the subtraction of $\mu_K^2uu\topr$ for centered covariance.

For $K\geq2$, integration by parts, using $z^2\phi(z)=\phi(z)-(z\phi(z))'$, gives
\[
\E M_K^2=K\int z^2\phi(z)\Phi(z)^{K-1}\,dz
=1+K(K-1)\int z\phi(z)^2\Phi(z)^{K-2}\,dz.
\]
For $K=2$, the final integrand is odd. For $K\geq3$, using $(\phi^2)'=-2z\phi^2$ and integrating again proves \eqref{eq:beta-integral}. All boundary terms vanish by Gaussian tails. The case $K=1$ is immediate. Finally,
$\E\max(Z_1,Z_2)=\tfrac12\E|Z_1-Z_2|=\pi^{-1/2}$, and evaluating $3\int\phi^3$ gives $\beta_3=\sqrt3/(2\pi)$.
\end{proof}

\subsection{Proof of Corollary~\ref{cor:two}}
\begin{proof}
For $K=2$, the conditional density in Proposition~\ref{prop:winner} is $p_u(y)=2\phi_d(y)\Phi(u\topr y)$. Since $\Phi(s)+\Phi(-s)=1$, it satisfies
$p_u(y)+p_u(-y)=2\phi_d(y)$. If $F$ is any bounded measurable function of a symmetric matrix, then $F(yy\topr)$ is even in $y$. Integrating against the symmetrized density shows that the conditional law of $d x_tx_t\topr$ equals the law of $gg\topr$, independently of $u$ and of the history. Iterated conditioning proves independence of these matrices across rounds. Their sum therefore has the claimed Wishart law. For $K=1$, no selection occurs.
\end{proof}

\subsection{Proof of Theorem~\ref{thm:posterior}}
\begin{proof}
\textbf{Posterior distribution.}
For a realized observed transcript, the density of every candidate pool and the conditional probability of every action contribute factors that do not depend on $\theta^\star$. The same holds for the independent policy seeds. The only parameter-dependent factors are the prior density and
$\prod_{s\leq t}\exp\{-(Y_s-x_s\topr\theta)^2/(2\sigma^2)\}$. Completing the square gives the Gaussian posterior with the objects in \eqref{eq:posterior-objects}. In particular, $0\prec C_t\preceq\tau^2I_d$.

\textbf{Covariance traces.}
The functions $f(x)=(\tau^{-2}+\sigma^{-2}x)^{-1}$ and $f(x)^2$ are bounded and Lipschitz on $[0,\infty)$. Theorem~\ref{thm:bulk} and Corollary~\ref{cor:resolvent} give the pointwise limits $v(s)=\sigma^2h_s(\sigma^2/\tau^2)$ and $w(s)$. The explicit formula in \eqref{eq:posterior-limits} and the polynomial \eqref{eq:v-polynomial} follow by substitution in \eqref{eq:transform-equations}.

To obtain uniformity in time, couple all prefixes through $N=\lfloor\Gamma d\rfloor$ as in Theorem~\ref{thm:coupling}. The expected supremum of the normalized trace differences between the adaptive and Gaussian matrices is bounded by the Lipschitz constant of $f$ times the right side of \eqref{eq:nuclear-bound}, which vanishes uniformly over policies. For the Gaussian comparison process, $d^{-1}\tr f(Z_{\lfloor sd\rfloor}Z_{\lfloor sd\rfloor}\topr)$ is nonincreasing in $s$. Its pointwise limits are the continuous function $v(s)$. A finite partition of $[0,\Gamma]$, monotonicity between partition points, and then uniform continuity of $v$ prove convergence of the supremum in probability. Since these traces and $v$ lie in $[0,\tau^2]$, the convergence holds in expectation. This proves \eqref{eq:uniform-posterior}.

\textbf{Realized squared error.}
Conditional on $\F_n$, the error is $N(0,C_n)$. Hence its squared norm has conditional mean $\tr C_n$ and conditional variance $2\tr C_n^2\leq2d\tau^4$. Therefore
\[
\E\left|\frac{\|\theta^\star-m_n\|^2-\tr C_n}{d}\right|
\leq\frac{\sqrt2\tau^2}{\sqrt d}.
\]
Combining this bound with the trace limit proves the assertion uniformly over policies.

\textbf{Information.}
The entropy of the Gaussian prior is $\tfrac d2\log(2\pi e\tau^2)$; the conditional posterior entropy is $\tfrac12\log\det(2\pi e C_n)$. Taking the difference in expectation proves the identity in \eqref{eq:information}. The function $x\mapsto\tfrac12\log(1+\tau^2x/\sigma^2)$ is Lipschitz on $[0,\infty)$. The expected $W_1$ limit therefore proves the information limit. Positivity of $v$ and its value at zero follow from the positive-root characterization.
\end{proof}

\subsection{Proof of Corollary~\ref{cor:posterior-normal}}
\begin{proof}
Given $\F_n$, diagonalize $C_n$ and denote its positive eigenvalues by $c_1,\ldots,c_d$. The centered squared error is distributed as $\sum_i c_i(Z_i^2-1)$ for independent standard normals. The Berry--Esseen inequality for independent summands bounds the distributional approximation error by
\[
C\frac{\sum_i c_i^3}{(\sum_i c_i^2)^{3/2}}.
\]
Here the universal constant absorbs the finite third absolute moment of $Z_i^2-1$ and the variance factor. This is the asserted conditional bound. Since $c_i\leq\tau^2$, the ratio is at most $\tau^2/(\tr C_n^2)^{1/2}$. Theorem~\ref{thm:posterior} gives $d^{-1}\tr C_{n_d}^2\to w(\gamma)>0$ in probability, proving the rate.
\end{proof}

\section{Proofs of the directional results}\label{app:outliers}
The proofs distinguish a direction that persists through the experiment from one that changes among orthogonal coordinates.

\subsection{Proof of Theorem~\ref{thm:outlier}}
\begin{proof}
Rotate coordinates so that $u_d=e_1$. Proposition~\ref{prop:winner} gives the distributional representation
\[
X_n=\frac1{\sqrt d}\begin{pmatrix}w\topr\\ Z\end{pmatrix},
\]
where the entries of $w\in\R^n$ are independent copies of $M_K$, the entries of the $(d-1)\times n$ matrix $Z$ are independent standard normals, and $w$ and $Z$ are independent. In particular, $\|w\|^2/n\to1+\beta_K$ in probability. The nonzero mean of $w$ does not invalidate this representation.

Let $B=ZZ\topr/d$. Its largest two eigenvalues converge to $b_\gamma$, and its empirical spectral distribution, normalized by $d$, has limit $\mu_\gamma$ up to the immaterial missing coordinate. These are the Gaussian sample covariance and spectral-edge limits; see \citet{BaiSilverstein2010}.

For fixed $\lambda>b_\gamma$, define
\[
Q_n(\lambda)=\frac1{d^2}w\topr Z\topr(\lambda I-B)^{-1}Zw,
\qquad
J_n(\lambda)=\frac1{d^2}w\topr Z\topr(\lambda I-B)^{-2}Zw.
\]
Right-orthogonal invariance of $Z$ permits the direction $w/\|w\|$ in these quadratic forms to be replaced, conditionally on its norm and on the singular values of $Z$, by a uniform unit vector in $\R^n$. For a symmetric matrix $H$ and such a vector $v$,
\[
\E(v\topr Hv)=\tr H/n,
\qquad
\Var(v\topr Hv)\leq\frac{2\|H\|_{\op}^2}{n}.
\]
The matrices $d^{-1}Z\topr(\lambda I-B)^{-k}Z$, for $k=1,2$, have bounded operator norm with probability tending to one when $\lambda$ stays a positive distance above $b_\gamma$. Spectral convergence and the preceding conditional variance bound therefore give
\begin{eqn}\label{eq:schur-limits}
Q_n(\lambda)\pto(1+\beta_K)\int\frac{x}{\lambda-x}\,d\mu_\gamma(x),
\qquad
J_n(\lambda)\pto(1+\beta_K)\int\frac{x}{(\lambda-x)^2}\,d\mu_\gamma(x).
\end{eqn}
The convergence is uniform on compact intervals above the edge, by a finite net and resolvent derivative bounds.

The Schur complement shows that the largest eigenvalue of $G_n$ is the unique root above $\lambda_1(B)$ of
\begin{eqn}\label{eq:schur-equation}
D_n(\lambda):=\lambda-\|w\|^2/d-Q_n(\lambda)=0.
\end{eqn}
Almost surely the off-diagonal block has nonzero projection on the leading eigenvector of $B$, so this root is strictly above $\lambda_1(B)$. Also, $D_n$ is strictly increasing there. Its deterministic limit is
\[
D(\lambda)=\lambda-(1+\beta_K)
\left\{\gamma+\int\frac{x}{\lambda-x}\,d\mu_\gamma(x)\right\}.
\]
The matrices have bounded operator norm in probability, since
$\|X_n\|_{\op}\leq\|Z\|_{\op}/\sqrt d+\|w\|/\sqrt d$.
Thus roots cannot escape to infinity.

Write $g_\gamma(\lambda)=\int(\lambda-x)^{-1}\,d\mu_\gamma(x)$. The transform identity gives $g_\gamma(b_\gamma)=1/(1+\sqrt\gamma)$ and
$\int x/(b_\gamma-x)\,d\mu_\gamma(x)=\sqrt\gamma$. It follows that
\[
D(b_\gamma+)=(1+\sqrt\gamma)(1-\beta_K\sqrt\gamma).
\]
Since $D$ is strictly increasing above the edge, it has a root there exactly when $\beta_K\sqrt\gamma>1$. To evaluate it explicitly, for $r>\gamma^{-1/2}$ put $\lambda(r)=(1+r)(\gamma+r^{-1})$. The positive-side transform equation and its branch at infinity give
\[
g_\gamma\{\lambda(r)\}=\frac1{1+\gamma r},
\qquad
\int\frac{x}{\lambda(r)-x}\,d\mu_\gamma(x)
=\lambda(r)g_\gamma\{\lambda(r)\}-1=\frac1r.
\]
Taking $r=\beta_K$ solves $D(\lambda)=0$, giving the root $(1+\beta_K)(\gamma+\beta_K^{-1})$. Uniform convergence in \eqref{eq:schur-limits} then proves the supercritical eigenvalue limit. In the other case, $D(b_\gamma+\epsilon)>0$ for every $\epsilon>0$, including at the critical threshold. Equation~\eqref{eq:schur-equation}, monotonicity, and $\lambda_1(G_n)\geq\lambda_1(B)$ imply convergence to $b_\gamma$.

The block eigenvector equation gives the exact identity
\begin{eqn}\label{eq:overlap-schur}
|e_1\topr v_1(G_n)|^2=\{1+J_n(\lambda_1(G_n))\}^{-1}.
\end{eqn}
In the supercritical case, uniform convergence away from the edge applies at the random eigenvalue. Differentiating the preceding identity $\int x/\{\lambda(r)-x\}\,d\mu_\gamma(x)=1/r$, and using $\lambda'(r)=\gamma-r^{-2}$, yields
\[
\int\frac{x}{\{\lambda(r)-x\}^2}\,d\mu_\gamma(x)
=\frac1{\gamma r^2-1}.
\]
Consequently,
\[
\left\{1+(1+\beta_K)\int\frac{x}{(\lambda-x)^2}\,d\mu_\gamma(x)\right\}^{-1}
=\frac{1-(\gamma\beta_K^2)^{-1}}{1+(\gamma\beta_K)^{-1}}
\]
at $\lambda=(1+\beta_K)(\gamma+\beta_K^{-1})$.

In the subcritical and critical cases, fix $\epsilon>0$. With probability tending to one, $\lambda_1(G_n)<b_\gamma+\epsilon$ while both values are above $\lambda_1(B)$. Hence $J_n(\lambda_1(G_n))\geq J_n(b_\gamma+\epsilon)$. The second integral in \eqref{eq:schur-limits} diverges as $\epsilon\downarrow0$: the density has square-root decay at the positive upper edge, so its product with $(b_\gamma+\epsilon-x)^{-2}$ has an integral tending to infinity. Equation~\eqref{eq:overlap-schur} proves vanishing overlap.

Finally, principal-submatrix interlacing gives
$\lambda_1(B)\geq\lambda_2(G_n)\geq\lambda_2(B)$, proving the second-eigenvalue limit. The numerical threshold for $K=3$ follows from Proposition~\ref{prop:winner}.
\end{proof}

\subsection{Proof of Proposition~\ref{prop:overlap}}
\begin{proof}
For the first policy, take $u_t=e_2$ for every $t$. Its outlier limit follows from Theorem~\ref{thm:outlier}. For the second policy, set
$u_t=e_{2+((t-1)\bmod(d-1))}$. Both policies are orthogonal to $e_1$ at every round.

Use the exact representation \eqref{eq:winner-representation} to couple the second selected design with independent standard Gaussian columns $g_t/\sqrt d$. The difference column is
$d^{-1/2}(M_{K,t}-u_t\topr g_t)u_t$. Each coordinate appears at most a constant number of times over $\lfloor\gamma d\rfloor$ rounds, and different rows of the difference matrix have disjoint column supports. Therefore
\[
\|X_n-Z_n\|_{\op}^2
=\max_{2\leq j\leq d}\frac1d
\sum_{t:u_t=e_j}(M_{K,t}-g_{t,j})^2
=O_\PP(\log d/d).
\]
The last bound follows from Gaussian tails and a union bound: a fixed Gaussian maximum and an independent Gaussian have a finite exponential moment of their squared difference for a sufficiently small positive exponent. Since $\|Z_n\|_{\op}=O_\PP(1)$, the Gram matrices differ by $o_\PP(1)$ in operator norm. The Gaussian upper-edge limit proves the claim for the second policy. Their common empirical spectral law follows from Theorem~\ref{thm:bulk}.
\end{proof}

\section{Proof of Theorem~\ref{thm:exponential}}\label{app:largepool}
\begin{proof}
For a single candidate, $\|x_{t,a}\|^2$ has law $\chi_d^2/d$. For every fixed $r>1$, its upper-tail probability satisfies
\begin{eqn}\label{eq:chisq-rate}
\frac1d\log\PP(\chi_d^2/d\geq r)\longrightarrow
-\frac{r-1-\log r}{2}.
\end{eqn}
The upper bound follows by Chernoff's inequality and the moment generating function $(1-2s)^{-d/2}$. For completeness, the matching lower bound follows by changing measure so that the underlying independent Gaussian coordinates have variance $r'>r$. On the event that their average squared value is within a fixed small interval around $r'$, the likelihood ratio has exponential rate $-(r'-1-\log r')/2$; this event has probability tending to one under the changed measure. Letting $r'\downarrow r$ proves \eqref{eq:chisq-rate}.

Let $S_d$ be the largest squared norm in one pool. For any $\epsilon>0$, a union bound and \eqref{eq:chisq-rate} imply $\PP(S_d>r_\alpha+\epsilon)\to0$. Conversely,
\[
\PP(S_d\leq r_\alpha-\epsilon)
\leq\exp\{-K_d\PP(\chi_d^2/d>r_\alpha-\epsilon)\}\to0
\]
whenever $r_\alpha-\epsilon>1$. The other cases follow by lowering the comparison point to any number in $(1,r_\alpha)$. Thus $S_d\to r_\alpha$ in probability.

The same Chernoff bound controls all fixed moments of $S_d$ uniformly in $d$. Indeed, for sufficiently large $r$, $(r-1-\log r)/2\geq r/4$, and the union bound gives an exponentially decaying tail beyond a constant depending only on $\alpha$. Consequently $S_d\to r_\alpha$ in $L^1$, and $\sqrt{S_d}\to\sqrt{r_\alpha}$ in $L^2$.

Rotational invariance implies that the selected vector has the representation $\sqrt{S_d}\,v$, where $v$ is uniform on the unit sphere and independent of $S_d$. Different rounds give independent copies. Introduce an independent $Q_d\sim\chi_d^2/d$ and couple the vector with $\sqrt{Q_d}\,v$, which is $N(0,I_d/d)$. Then
\[
\E\bigl\|\sqrt{S_d}\,v-\sqrt{r_\alpha Q_d}\,v\bigr\|^2
=\E(\sqrt{S_d}-\sqrt{r_\alpha Q_d})^2\longrightarrow0.
\]
For the resulting design matrices, $\E\|X_n-\sqrt{r_\alpha}Z_n\|_F^2=o(n)$. The nuclear-norm argument in the proof of Theorem~\ref{thm:coupling}, now comparing with $\sqrt{r_\alpha}Z_n$, gives
$d^{-1}\E\|G_n-r_\alpha Z_nZ_n\topr\|_*\to0$. The Gaussian Marchenko--Pastur law, its first-moment convergence, and the Wasserstein triangle inequality finish the proof.
\end{proof}

\bibliographystyle{plainnat}
\bibliography{references}
\end{document}